%% file: nmfSync.tex
\documentclass[preprint,3p,times,twocolumn,nopreprintline]{elsarticle} %
\usepackage{amssymb}  
\usepackage{floatflt}
\usepackage{amsmath}
\usepackage{mathtools}
\usepackage{amsthm}

\usepackage{graphicx}

\usepackage{dsfont}
\usepackage{enumitem}
\usepackage{wrapfig}

\usepackage[tight,scriptsize,sf,SF]{subfigure}

\usepackage[linesnumbered, ruled, vlined,noend]{algorithm2e}  %
 \SetKwInput{KwInput}{Input}
 \SetKwInput{KwOutput}{Output} 

\usepackage[dvipsnames]{xcolor}

\input{macros}

\journal{Pattern Recognition}

\makeatletter
\def\ps@pprintTitle{%
 \let\@oddhead\@empty
 \let\@evenhead\@empty
 \def\@oddfoot{}%
 \let\@evenfoot\@oddfoot}
\makeatother

\begin{document}

\begin{frontmatter}

\title{Synchronisation of Partial Multi-Matchings via Non-negative Factorisations}

\author{Florian Bernard$^{1}$, Johan Thunberg$^{2,3}$, Jorge Goncalves$^{3}$, Christian Theobalt$^{1}$}

\address{$^1$MPI Informatics, Saarland Informatics Campus ~ $^2$Halmstad University ~ $^3$LCSB Luxembourg}

\begin{abstract}
In this work we study permutation synchronisation for the challenging case of partial permutations, which plays an important role for the problem of matching multiple objects (e.g.~images or shapes). The term synchronisation refers to the property that the set of pairwise matchings is cycle-consistent, i.e.~in the full matching case all compositions of pairwise matchings over cycles must be equal to the identity. Motivated by clustering and matrix factorisation perspectives of cycle-consistency, we derive an algorithm to tackle the permutation synchronisation problem based on non-negative factorisations. In order to deal with the inherent non-convexity of the permutation synchronisation problem, we use an initialisation procedure based on a novel rotation scheme applied to the solution of the spectral relaxation. Moreover, this rotation scheme facilitates a convenient Euclidean projection to obtain a binary solution after solving our relaxed problem. In contrast to state-of-the-art methods, our approach is guaranteed to produce cycle-consistent results. We experimentally demonstrate the efficacy of our method and show that it achieves better results compared to existing methods.
\end{abstract}


\end{frontmatter}
\section{Introduction}
The problem of matching features across images or shapes is a fundamental topic in pattern recognition and vision
and has a high relevance
in a wide range of problems. Potential applications include shape deformation model learning \cite{Cootes:1992uw,Heimann:2009kv},
object tracking, 3D reconstruction, graph matching, %
or image registration. The fact that many tasks that seek for a  matching between a pair of objects can be formulated as the NP-hard quadratic assignment problem (QAP) \cite{Sahni:1976gt} illustrates the difficulty of matching problems. 
The more general problem of matching an entire collection of objects, rather than a pair of objects, is referred to as
\emph{multi-matching}. In general, such multi-matching problems are computationally at least as difficult as pairwise matching problems, as they can be phrased in terms of simultaneously solving multiple pairwise matching problems that are coupled via consistency constraints. Using such couplings of pairwise problems is a common approach for solving multi-matching problem in practice \cite{Kezurer:2015,Yan:2015vc,Yan:2016vf,Bernard:2018}.

Due to the importance and practical relevance of making use of pairwise matchings to solve multi-matching problems,
in this work we focus on studying \emph{permutation synchronisation} methods. The aim of these methods is to process a given set of ``noisy'' pairwise matchings such that \emph{cycle-consistency} is achieved. In the case of full matchings, cycle-consistency refers to the property that compositions of pairwise matchings over cycles must be equal to the identity matching.
Synchronisation methods have been studied extensively both in the context of multi-matching (e.g.~\cite{Nguyen:2011eb,Pachauri:2013wx,Huang:2013uk, Chen:2014uo,Shen:2016wx,Tron:kUBrCZhd,Maset:YO8y6VRb,Schiavinato:2017fr}) as well as for general transformations (e.g.~\cite{Govindu:2004jx,Hadani:2011hb,Singer:2011ba,Chatterjee:2013vi,Bernard_2015_CVPR,Arrigoni:2017ut,Thunberg:2017kka,Wang:2013tq}).
One can interpret the synchronisation methods as a denoising procedure, where the wrong matchings (i.e.~the noise) that account for cycle inconsistencies in the set of pairwise matchings are to be filtered out.

Most commonly, the synchronisation of pairwise matchings is formulated as an optimisation problem over permutation matrices. In the works by Pachauri et al. \cite{Pachauri:2013wx} and Shen et al. \cite{Shen:2016wx}, solutions for the synchronisation of permutation matrices based on a spectral factorisation are presented. A major limitation of these works is %
that the method is only suitable for \emph{full} permutation matrices, i.e.~it is assumed that all features are present in all objects (cf.~Sec.~\ref{sec:existingSpectral}). While this limitation has recently been addressed in the work by Maset et al.~\cite{Maset:YO8y6VRb},
in their work they do not aim for cycle-consistency. Since the (unknown) true matchings must be cycle-consistent, we argue that cycle-consistency is essential and should be strived for.

The main objective of this work is to present a novel approach for the synchronisation of pairwise matchings that addresses the mentioned shortcomings of existing methods. 
To this end, we present an improved formulation for the permutation synchronisation problem that finds a non-negative approximation of the range space of the pairwise matching matrix.
In contrast to \cite{Pachauri:2013wx}, our approach can handle \emph{partial} pairwise matchings. %
 Moreover, unlike \cite{Maset:YO8y6VRb,zhou2015multi}, our approach guarantees cycle-consistent matchings.

\paragraph{Main contributions:} The main contributions of our work on the sychronisation of partial permutations can be summarised as follows:
  (i) Motivated by clustering and matrix factorisation perspectives of cycle-consistency in the set of pairwise matchings,
  we derive an \emph{improved algorithm for permutation synchronisation} based on non-negative factorisations.
   (ii) While the proposed formulation is non-convex, we propose a \emph{novel procedure for initialising} the variables. 
  (iii) Moreover, we present a  \emph{novel projection procedure} to obtain a binary solution from the relaxed formulation. %
  (iv) Experimentally we demonstrate that our method achieves \emph{superior results} on synthetic and real datasets, while addressing the aforementioned shortcomings.

\section{Related Work}\label{sec:rel}
In this section we discuss prior work that is most relevant to our  approach.

\paragraph{Transformation synchronisation:} 
Synchronisation methods have been studied for various kinds of transformations. The synchronisation of (special) orthogonal transformations has been considered based on spectral methods \cite{Singer:2011ba,Bandeira:2013up,Wang:2013tq}, semidefinite programming \cite{Singer:2011ba,Chaudhury:2013un,Wang:2013tq}, or Lie-group averaging \cite{Govindu:2004jx,Chatterjee:2013vi}. The case of rigid-body transformations, which is particularly relevant in the context of vision, has been studied in semidefinite programming frameworks \cite{Chaudhury:2013un,Bandeira:2014wy}, as well as in the context of spectral approaches \cite{Bernard_2015_CVPR,Arrigoni:2016cc}. In general, spectral approaches are more scalable compared to semidefinite programming methods.
In addition to centralised methods, distributed synchronisation methods have also been presented, both for the case of undirected graphs \cite{tron2014distributed}, as well as for the more general case of directed graphs \cite{Thunberg:2017kka}.

\paragraph{Permutation synchronisation:}
Since permutation matrices are a subset of the orthogonal matrices, one could consider permutation synchronisation as a special case of the orthogonal synchronisation methods. However, in general the permutation synchronisation problem appears to be more difficult due to the additional binary constraints. Moreover, if one considers partial permutations, this interpretation as special case is no longer valid.
The synchronisation of full permutation matrices has been presented by Pachauri et al.~\cite{Pachauri:2013wx}, with follow-up works that consider partial matchings \cite{Arrigoni:2017ut,Maset:YO8y6VRb}. We devote Sec.~\ref{sec:existingSpectral} to an in-depth explanation of these approaches, where we also identify their main weaknesses upon which our approach improves.

\paragraph{Matching problems:}
Matching problems between two objects are commonly formulated in terms of the linear assignment problem (LAP) \cite{Burkard:2009hp,Munkres:1957ju} or the quadratic assignment problem (QAP) \cite{Koopmans:1957gf,Lawler:1963wn,Burkard:2009hp,Loiola:2007}. When one matches graphs, the LAP corresponds to matching node attributes only, whereas the QAP matches node attributes as well as edge attributes \cite{Zhou:2013ty}.
Computationally, the difference between both is that the LAP is solvable in polynomial time (e.g.~via the Hungarian method \cite{Munkres:1957ju} or the Auction algorithm \cite{Bertsekas:1998vt}), whereas the QAP is NP-hard \cite{Sahni:1976gt}. Hence, for solving QAPs in practice, existing approaches either resort to (expensive) branch and bound methods \cite{Bazaraa:1979fh}, or to approximations, e.g.~based on spectral methods \cite{Leordeanu:2005ur,Cour:2006un}, dual decomposition \cite{Torresani:2013gj}, linear relaxations \cite{swoboda2017a,swoboda2017b}, convex relaxations \cite{Zhao:1998wc,Schellewald:2005up,Fogel:2013wta,Fiori:2015cw,Kezurer:2015,Aflalo:2015hda,Dym:2017ue,Bernard:2018}, path following \cite{Zaslavskiy:2009fq,Zhou:2013ty,Jianga:vg}, or alternating directions \cite{LeHuu:2016uq}.

\paragraph{Multi-matching problems:}
The problem of matching more than two objects can be phrased as multi-graph matching (MGM) problems \cite{Williams:1997vj,Yan:2013ve,Yan:2015wr,Huang:2013uk,Yan:2015wr,Kezurer:2015,Shi:2016tj,Bernard:2018,Hu:2018vb}, which in general are computationally very challenging. If one uses first-order terms only, so that geometric relations between the features are not explicitly taken into account, multi-matching can efficiently be solved as (constrained) clustering problem \cite{yan2016constrained,Tron:kUBrCZhd}. 
The approaches described in \cite{Kezurer:2015,Yan:2015vc,Yan:2016vf,Bernard:2018} phrase MGM in terms of multiple pairwise matchings. The work in \cite{zhou2015multi} is closely related to the permutation synchronisation methods \cite{Pachauri:2013wx,Arrigoni:2017ut,Maset:YO8y6VRb}, as the authors formulate the multi-matching problem directly in terms of a low-rank optimisation problem for a given set of pairwise matchings. However, the so-obtained matchings are generally not cycle-consistent.

\section{Background}

\paragraph{Notation:} Let $\onevec_{pq}$ and $\zerovec_{pq}$ denote $p \times q$ matrices comprising of ones and zeros, and we write $\onevec_p$ and $\zerovec_p$ for $q=1$.
We use $X_+$ to denote that all negative elements in the matrix $X$ are replaced by $0$.  For an integer $i \in \N$, we define $[i] := \{1,\ldots,i\}$. 
For a $p \times q$ matrix $X$, and the index sets $A \subseteq [p], B \subseteq [q]$, we denote by $X_{A,B}$ the $\vert A \vert \times \vert B \vert$ submatrix of $X$ that is formed from the rows with indices in $A$ and the columns with indices in $B$. We use the colon notation to denote the full index set, e.g.~$X_{:,B} = X_{A,B}$ for $A = [p]$. For matrices $A_{ij}, i \in [p], j \in [q]$ of appropriate sizes, we use the shorthand notation $[A_{ij}]_{ij}$ to denote the block matrix
\begin{align}
	[A_{ij}]_{ij} := \begin{bmatrix}
		A_{11} & \hdots & A_{1q} \\
		\vdots & \ddots & \vdots \\
		A_{p1} & \hdots & A_{pq}
	\end{bmatrix}\,.
\end{align}
The set of (full) permutation matrices is defined as
\begin{align}
  \perm_{p} := \{ X \in \{0,1 \}^{p \times p} ~:~ X \onevec_p = \onevec_p, \onevec_p^T X  = \onevec_p^T\} \,.
\end{align}
The set of $p \times q$ partial permutation matrices $\perm_{pq}$ is defined as
\begin{align}\label{eq:partperm}
  \perm_{pq} := \{ X \in \{0,1 \}^{p \times q} ~:~ X \onevec_q \leq \onevec_p, \onevec_p^T X \leq \onevec_q^T\} \,.
\end{align}

\subsection{Partial Permutation Synchronisation}
Let $k \in \N, k > 2$ be the total number of objects (e.g.~images or shapes) that are to be matched. We assume that in object $i \in \N$, with $i \in [k]$, there are $m_i \in \N$ features, where the total number of features is denoted as $m = \sum_{i=1}^k m_i$.
 Moreover, we assume that there is a total number of $d \in \N$ distinct features across all objects $i \in [k]$ in the \emph{universe}.   
We use $P_{ij} \in \perm_{m_im_j}$ to denote a (partial) permutation   that encodes the matching between the $i$-th and the $j$-th object (Fig.~\ref{fig:illustration}(i)). The element $(P_{ij})_{pq} \in \{0,1\}$ at position $(p,q)$, $p \in [m_i], q \in [m_j]$ of matrix $P_{ij}$ is $1$ iff the $p$-th feature of object $i$ is matched to the $q$-th feature of object $j$.
For $P_{ij} \in \perm_{m_i m_j}$, $W := [P_{ij}]_{i,j \in [k]} \in [\perm_{m_i m_j}]_{i,j \in [k]}$ is the $m {\times} m$ matrix of pairwise (partial) matchings. 

\paragraph{Cycle-consistency of partial matchings:} In contrast to full matchings, where cycle-consistency refers to the property that compositions of pairwise matchings over cycles must be \emph{equal} to the identity matching, in the case of partial matchings one only requires that compositions of pairwise matchings over cycles must be \emph{a subset of the identity matching}. Due to potential pairwise \emph{non-matchings} (i.e.~zero rows or columns in $P_{ij}$) along a cyclic path, some of the original matchings may vanish. A convenient way to define \emph{cycle-consistency} of partial matchings is based on universe features:
\begin{definition} \label{def:transconspartial}
  The matrix of pairwise (partial) matchings $W = [P_{ij}]_{i,j \in [k]}$
  is said to be \emph{cycle-consistent}  (or \emph{synchronised}) iff there exists a set $\{P_i \in \perm_{m_id}: i \in [k],  P_i \onevec_d = \onevec_{m_i} \}$ such that for all $i,j \in [k]$ it holds that $P_{ij} =  P_i P_j^T$. 
\end{definition}
The {\emph{object-to-universe matching}} matrices $P_i \in \perm_{m_id}$ can be interpreted as assignments of each feature of the $i$-th object to one of the features in the universe (Fig.~\ref{fig:illustration}(ii)), where the $\ell$-th row of $P_i$ is the assignment of the $\ell$-th feature of object $i$ to a particular feature in the universe.
The requirement $P_i \onevec_d = \onevec_{m_i}$ ensures that each feature of object $i$ is assigned to exactly one feature of the universe. 
\begin{figure*}[t!]
\centering
  \includegraphics[width=0.9\linewidth]{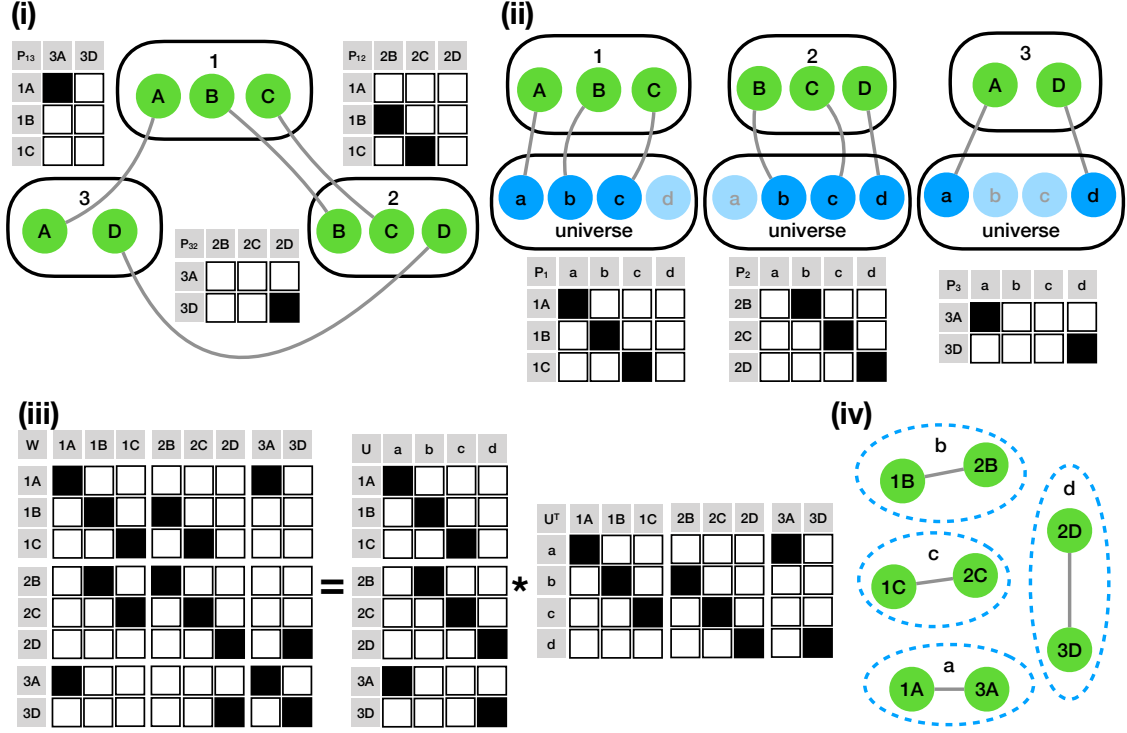}
  \caption{Conceptual illustration  of (i) relative matchings, (ii) absolute matchings, (iii) the matrix factorisation perspective, and (iv) the graph of pairwise matchings. The objects are denoted by $1$, $2$ and $3$, where corresponding features across objects are labelled by the same letter from A to D. The \emph{relative} matchings are represented by the permutation matrices $P_{ij}$ (e.g.~$P_{12}$), and the \emph{absolute} matchings are represented by the permutation matrices $P_i$ (e.g.~$P_1$) that match each feature to one of the universe features a, b, c, d. Since cycle-consistency holds in this case, the matrix $W$ in (iii) can be factorised into $UU^T$ (Lemma~\ref{lem:transconsmatrix}); and the graph of pairwise matchings in (iv) is a union of the disconnected cliques a, b, c and d (Lemma~\ref{lem:graph}).}
  \label{fig:illustration} 
\end{figure*}
For $\mathcal{U} :=\{U \in [\perm_{m_i d}]_{i\in [k]} ~:~U \onevec_d = \onevec_m\} \subset \R^{m \times d}$, %
one can characterise cycle-consistency of partial matchings in terms of a low-rank factorisation \cite{Maset:YO8y6VRb}, which is also illustrated in Fig.~\ref{fig:illustration}(iii):
\begin{lemma} \label{lem:transconsmatrix} %
The pairwise (partial) matching matrix $W$ is cycle-consistent iff there exists a matrix $U \in \mathcal{U}$, such that $W = UU^T$.
\end{lemma}  
\begin{proof}
To prove the statement we identify $U = \begin{bmatrix}
    P_1^T & P_2^T &\cdots & P_k^T
  \end{bmatrix}^T \in \R^{m \times d}$. One can easily see, cf.~Def.~\ref{def:transconspartial}, that cycle-consistency implies that there exists a $U$ that has the desired properties.
 Likewise, if a $U \in \mathcal{U}$ with $W=UU^T$ is given, one can see that the blocks $\{P_i \}$ of $U$ satisfy $P_i \onevec_d = \onevec_{m_i}$ as well as $P_{ij} = P_iP_j^T$. %
\end{proof}

\paragraph{Optimisation problem:}  %
Lemma~\ref{lem:transconsmatrix} shows that in the noise-free case, the matrix of pairwise matchings $W$ can be factorised as $W=UU^T$. %
Given a \emph{noisy} $W$, a straightfoward way to phrase the permutation synchronisation problem is to consider the constrained nonlinear least-squares problem
\begin{align} \label{eq:syncOpt}
  \argmin_{U \in \mathcal{U}} \| W{-}UU^T \|_F^2 \,.
\end{align}
Since Problem~\eqref{eq:syncOpt} is non-convex, %
finding an exact solution is intractable for reasonably large instances. Hence, various simplifications  have been considered in the literature, as we describe next.

\subsection{Spectral Relaxations}\label{sec:existingSpectral}
In this section we summarise the key ideas of existing spectral relaxations, where we also identify their shortcomings when synchronising \emph{partial} permutations.
In order to avoid confusion, we explicitly mention that the reader should carefully distinguish between the $d \times d$ matrix $U^TU$ and the $m \times m$ matrix $UU^T$, as both terms will appear below.

\paragraph{Full matchings:} In the case of (cycle-consistent) full matchings,  it holds that $U^T U = k \matI_d$. Thus, $\| W{-}UU^T \|_F^2 = \langle W,W \rangle {-} 2\langle W, UU^T \rangle {+} \langle UU^T,UU^T \rangle = \text{const}{-} 2\langle W, UU^T \rangle  $. Hence, for  full matchings, the authors of \cite{Pachauri:2013wx} relax the constraint $U \in \mathcal{U}$ to $U^TU {=} k\matI_d$, and then solve Problem~\eqref{eq:syncOpt} with the relaxed constraints by eigendecomposition, followed by a projection step.

\paragraph{Partial matchings:} For partial matchings, the authors of \cite{Maset:YO8y6VRb} propose to maximise $\langle W, UU^T \rangle $ based on eigendecomposition. However, in the \emph{partial} matchings case, in general  $U^TU {\neq} k\matI_d$, so that the objective $\langle W, UU^T \rangle $ differs from the objective in Problem~\eqref{eq:syncOpt}. 
Instead, for $U \in \mathcal{U}$ the objective $\langle W, UU^T \rangle$ counts the number of \emph{equal matchings} between the matrices $P_{ij}$ and $P_i P_j^T$ for all $i,j$.
A further difficulty with \emph{partial} matchings is related to the necessary projection due to the relaxation of the constraints, as we describe next.

\paragraph{Projection:} \label{sec:proj}
When the constraint $U \in \mathcal{U}$ is replaced by $U^T U = k \matI_d$, after obtaining $U$ based on the spectral decomposition of $W$, one needs to project $U$ onto the set $\mathcal{U}$. 
Since for any orthogonal matrix $Q \in \R^{d \times d}$ it holds that $(UQ)(UQ)^T {=} UQQ^TU = UU^T$, the factorisation $UU^T$ is only determined up to such a matrix $Q$. Hence, for projecting the blocks of $U$, one can choose a suitable orthogonal matrix $Q$ in order to simplify the projection. %
For the full matching case, the authors of \cite{Pachauri:2013wx} suggest to perform Euclidean projections of the $d \times d$ blocks of $U Q$ for the choice $Q = P_1^T$. Under the assumption that $W$ is relatively close to the form $UU^T$,  the matrix $P_1$ is near-orthogonal, such that the first block of $UQ$ is close to the identity matrix, while the remaining blocks of $UQ$ shall become close to permutation matrices.

Since for partial permutations the matrices $P_i$ are of dimension $m_i \times d$, where generally $m_i {<} d$, the assumption that the $P_i$ are near-orthogonal breaks, and thus such a procedure is not applicable anymore (cf.~Sec.~\ref{sec:projAfterRot} for details).
As workaround, instead of projecting the blocks of $U$ onto $\mathcal{U}$, the authors of \cite{Maset:YO8y6VRb} perform a projection of the blocks of $UU^T$, such that the $m \times m$ matrix $\proj(UU^T)$ is obtained. 
While it is reasonable (under small noise assumptions) to assume that the blocks of $UU^T$ are close to being (partial) permutation matrices, in this approach one cannot guarantee that the matrix $\proj(UU^T)$ satisfies the conditions in Lemma~\ref{lem:transconsmatrix}, and thus, cycle-consistency is violated. 

Another approach for the projection is pursued by the authors of \cite{zhou2015multi,Arrigoni:2017ut}, where a greedy strategy is employed for obtaining blocks of partial permutations from the matrix of eigenvectors $U$.

\subsection{Clustering Perspective}\label{sec:clustering}
Here, we summarise the clustering perspective of synchronisation (cf.~\cite{Arrigoni:2017ut,Tron:kUBrCZhd}), which %
will become useful to motivate our approach in Sec.~\ref{sec:ourapproach}.
For that, we consider the \emph{graph of pairwise matchings} $\G := \G(W)$ (cf.~Fig.~\ref{fig:illustration}(iv) for an illustration).  %
The (non-negative) $m \times m$ matrix $W$ is considered as the adjacency matrix of $\G$, so that $\G$ comprises $m$ nodes (recall that $m=\sum_i m_i$). The value $(W)_{pq} \in \R$ at position $(p,q)$ of $W$ denotes the edge weight that represents the affinity of nodes $p \in [m]$ and $q \in [m]$, where $(W)_{pq}=0$ means that there is no edge. Note that w.l.o.g. we assume $(W)_{pp}=1$ for all $p \in [m]$.
As shown by Tron et al. \cite{Tron:kUBrCZhd}, and illustrated in Fig.~\ref{fig:illustration}(iv), cycle-consistency can compactly be formulated in terms of the graph of pairwise matchings:
\begin{lemma} \label{lem:graph}
  The graph of pairwise matchings $\G(W)$ is cycle-consistent iff it is a union of disconnected cliques.
\end{lemma}
\begin{proof}
  See Prop.~2 in \cite{Tron:kUBrCZhd}. %
\end{proof}

\begin{lemma}\label{lem:colClust}
 Let the graph of pairwise matchings $\G(W)$ be cycle-consistent so that it is a union of the disconnected cliques $C_i \subseteq [m]$, $i \in [d]$. It holds that all columns of the matrix $W_{:,C_i} \in \{0,1\}^{m \times \vert C_i \vert}$ are equal for a given $i \in [d]$. 
\end{lemma}
\begin{proof}
\newcommand\mydots{\hbox to 1em{.\hss.\hss.}}
  We denote by $c_i$, $i \in [d]$, the number of elements in the $i$-th clique.
  Since $\G$ is a union of $d$ disconnected cliques, there is a permutation $P \in \perm_{m}$ such that $PWP^T$ is the block-diagonal matrix $PWP^T = \diag(\onevec_{c_1c_1}, \ldots,\onevec_{c_dc_d})$. Moreover, for $P$ it holds that
$\matI_{:,C_i} = P^T \matI_{:,A_i}$ for $A_i = \{d_i{+}1, d_i{+}2, \ldots, d_i{+}c_i\}$ with $d_i = \sum_{\ell = 1}^{i-1}c_{\ell}$. 
From $\matI_{:,C_i} = P^T \matI_{:,A_i}$ it follows that $(P W)_{:,C_i} = (P W)\matI_{:,C_i} = (PW P^T) \matI_{:,A_i} = \diag(\onevec_{c_1c_1}{,} \mydots{,}\onevec_{c_dc_d}) \matI_{:,A_i} {=} [\zerovec_{c_ic_1}^T{,}\mydots{,}\zerovec_{c_ic_{i-1}}^T{,} \onevec_{c_ic_i}^T{,} \zerovec_{c_ic_{i+1}}^T{,}\mydots{,} \zerovec_{c_ic_d}^T]^T$, which shows that the columns of $(P W)_{:,C_i}$ are equal. Hence, with $P W$ being a permutation of the rows of $ W$, the columns of $ W_{:,C_i}$ must also be equal.
Since cycle-consistency implies symmetry of $W$, the analogous statement also holds for the rows of $W$. 
\end{proof}

Lemma~\ref{lem:colClust} illustrates that one can cluster the columns (or rows) of $W$ to identify to which \emph{universe feature} they belong (cf.~Fig.~\ref{fig:illustration}(iv)).

\begin{figure} %
\centering
  \includegraphics[scale=1]{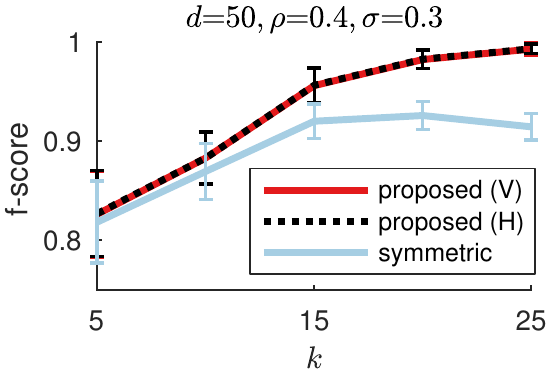}
  \vspace{-5mm}
  {\caption{Comparison of the proposed approach (when using either $V$ or $H$ to obtain the cycle-consistent matchings) with a \emph{symmetric} NMF~\cite{vandaele2016efficient}. While using an \emph{unsymmetric factorisation} is clearly advantageous, both $V$ and $H$ can be used equivalently.}
  \label{fig:symvsnonsym} 
  }
\end{figure}

\section{Proposed Approach}\label{sec:ourapproach}
A key idea of our approach is to formulate the permutation synchronisation problem in terms of a \emph{non-negative matrix factorisation} (NMF) \cite{Lee:uf}. To be more specific, we propose to solve 
\begin{align} \label{eq:nmfOpt}
  \argmin_{V \geq 0, H \geq 0} \| W{-} VH \|_F^2 \,,
\end{align}
where $V \in \R^{m \times d}$ and $H \in \R^{d \times m}$. 
Problem~\eqref{eq:nmfOpt} is a relaxation of Problem~\eqref{eq:syncOpt}, where the constraints $V=H^T$ are dropped, and the constraint set $\mathcal{U}$ is replaced by non-negativity constraints. 
At first sight it may appear unnatural that one aims for an \emph{unsymmetric} factorisation $VH$ of the symmetric matrix ${W}$.
However, we have found that this
is advantageous compared to a symmetric factorisation (see Fig.~\ref{fig:symvsnonsym}), which we believe is due to the following reasons: %
(i) On the one hand, from a theoretical perspective the factorisation $VH$ enables to get a better rank-$d$ approximation of ${W}$ (cf.~Lemma~\ref{lem:transconsmatrix}) compared to enforcing $H^T$ to be equal to $V$.
(ii) On the other hand, %
the unsymmetric NMF optimises over a higher-dimensional space, such that it has more freedom during the optimisation and is thus less prone to unwanted local optima of the non-convex Problem~\eqref{eq:nmfOpt}. 
(iii) Furthermore, with the inherent clustering properties of NMF \cite{Ding:2005ey,Li:2006bw,Ding:2008fo,Yang:2012ul,Lu:2014ht,Yang:2016ui}, %
Problem~\eqref{eq:nmfOpt} can also be understood from the clustering point-of-view
 (cf.~Sec.~\ref{sec:clustering}). %
In the clustering perspective, the columns of the matrix $V$ can be seen as the \emph{cluster centres}, where each column of ${W}$ is a conic combination of the columns of $V$, and the corresponding column of $H$ contains the coefficients. 
{%
Since swapping the roles of $V$ and $H$ is equivalent to factorising $W^T$ in place of $W$,
using either $V$ or $H$ for obtaining the cycle-consistent partial matchings from the \emph{unsymmetric factorisation} $VH$ are equivalent, as also demonstrated in~Fig.~\ref{fig:symvsnonsym}. Note that due to points (i) and (ii) it nevertheless is important that the factorisation is $unsymmmetric$ (cf.~Fig.~\ref{fig:symvsnonsym}).
}

The motivation for enforcing both $V$ and $H$ to be non-negative is as follows: when cycle-consistency holds, the columns of $V$ should be non-negative and mutually orthogonal, so that each row in $V$ can contain at most one non-zero element. Thus, if the factor matrix $H$ is such that $W = VH$, then, since $W$ is non-negative, $H$ needs to be non-negative. %

Next, we introduce our rotation scheme that is used for the initialisation of $V$ and $H$, as well as for the projection of $V$ onto $\mathcal{U}$.

\subsection{Rotation Scheme}\label{sec:projAfterRot}
For $X_i \in \R^{m_i \times d}$, $i \in [k]$, let $X = [X_1^T, \ldots, X_k^T]^T \in \R^{m \times d}$ be a rank-$d$ matrix that comprises a low-rank approximation of $W$, i.e.~$W \approx XX^T$. 
For any orthogonal matrix $Q$ we have that $XX^T = (XQ)(XQ)^T$, so that we can freely choose $Q$ and use $(XQ)(XQ)^T$ as low-rank approximation of $W$ in place of $XX^T$. The purpose of this section is to describe a procedure to find a $Q$, such that $XQ$ is closer to the set $\mathcal{U}$ compared to $X$, which is for example beneficial for performing a Euclidean projection of $X$ onto $\mathcal{U}$. 
To this end, we generalise the full-matching rotation scheme in \cite{Pachauri:2013wx}, which will be explained in the next paragraph, such that one can find a suitable orthogonal matrix $Q$ for the case of \emph{partial} matchings. 
\newcommand{\imScaleAA}{.4}
\newcommand{\imScale}{1}
\newcommand{\imSpace}{~~~~} 
\begin{figure*}[h!t!]%
  \centerline{%
         \subfigure[input]{\includegraphics[scale=\imScaleAA]{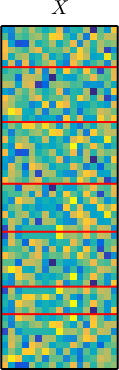}}\imSpace
         \subfigure[first iteration]{\includegraphics[scale=\imScaleAA]{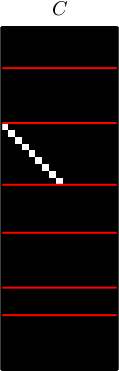}\includegraphics[scale=\imScaleAA]{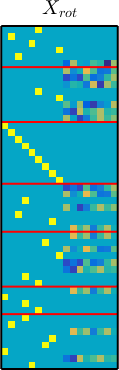}}\imSpace
         \subfigure[second iteration]{\includegraphics[scale=\imScaleAA]{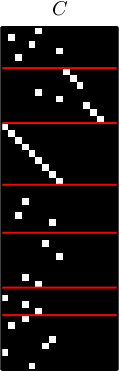}\includegraphics[scale=\imScaleAA]{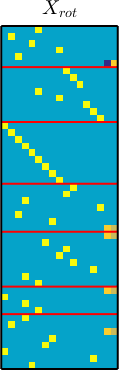}}\imSpace
         \subfigure[third iteration]{\includegraphics[scale=\imScaleAA]{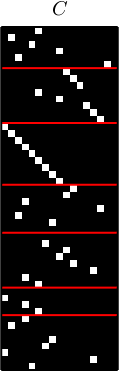}\includegraphics[scale=\imScaleAA]{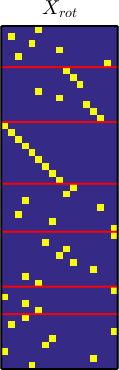}}
  }
\caption{{Illustration of the \emph{Successive Block Rotation Algorithm}. In each subimage the $k{=}7$ blocks are separated by red horizontal lines. (a) The input matrix $X$. (b) In the first iteration, the masking matrix $C$ is initialised so that the third block (which is the largest one) contains an $(m_3{\times}m_3)$-dimensional identity matrix. Solving Problem~\eqref{eq:Qsvd} results in the rotated $X_{\text{rot}}$. (c) In the second iteration, $C$ is updated such that on the one hand all active elements of the previous $X_{\text{rot}}$ remain active, and on the other hand all the inactive rows in the block with the largest number of inactive rows will be activated (in this case the second block). (d) The third iteration produces $X_{\text{rot}}$, where $X_{\text{rot}} \in \mathcal{U}$.}
\label{fig:sbra}
}
\end{figure*}

\paragraph{Challenges:} As discussed in Sec.~\ref{sec:existingSpectral}, in the case of full matchings, i.e.~$m_1 = \ldots = m_k = d$, the authors of \cite{Pachauri:2013wx} set $Q = X_i^T$ for one of the block indices $i \in [k]$, e.g.~$Q = X_1^T$, so that $XQ$ is close to a matrix that comprises blocks of permutations. 
This is based on the assumption that the pairwise matchings $XX^T$ are close to the ground truth, which in turn implies that (i) each $X_i^T$ is near-orthogonal, so that $X_iX_i^T \approx \matI_d \in \perm_d$ for all $i \in [k]$; and that (ii) there exists an orthogonal $Q \in \R^{d \times d}$ such that $XQ$ is close to comprising blocks of permutation matrices, so that
for all $ j \in [k] $ there exists a $P_j \in \perm_d$ such that $X_j X_i^T \approx P_j$. Essentially, due to (i) and (ii) it is ensured that $XQ$ is close to $\mathcal{U}$ whenever $Q = X_i^T$ for any $i \in [k]$. 

For partial matchings, point (i) is not valid anymore, because generally not \emph{all} the universe features are present in each object $i \in [k]$. Hence, the $X_i \in \R^{m_i \times d}$ are (generally) not orthogonal (as they are not even square matrices), from which it follows that $X X_i^T (X X_i^T)^T = XX_i^T X_i X^T \neq XX^T$. 
{For partial matchings, it is not sufficient to consider only a single block $X_i$ of $X$ for constructing $Q$. Instead, one needs to aggregate information from rows of $X$ that come from different blocks $X_1,\ldots,X_k$.
We tackle this using the \emph{Successive Block Rotation Algorithm} (\textsc{SBRA}), as we describe next.
}

{
\paragraph{Successive Block Rotation Algorithm:}  Similarly as in~\cite{Pachauri:2013wx}, we assume that a given $X$ forms a sufficiently good approximation $XX^T$ to the (unknown) ground truth matchings. With that, there must exist an orthogonal $Q$ such that $XQ$ is close to an element of $\mathcal{U}$, in which case each row of $XQ$ has a single element that is close to one, with all other elements being close to zero. When we make particular elements in $XQ$ close to one by rotating $X$ by $Q$, we say that we \emph{activate} these elements.
}

{
For finding a suitable orthogonal matrix $Q$ in the case of partial matchings, we successively select elements of $X$ that shall be activated.
Moreover, we ensure that at most one element in each row in $XQ$ is activated, so that all other elements in these rows become small (based on the above assumption).
To this end, we employ an $(m{\times}d)$-dimensional binary matrix $C$, which  has the purpose of masking those elements that shall become activated in the rotated $XQ$. 
For now, let us assume that we are given a $C\in \{0,1\}^{m{\times}d}$. With that, we consider the problem
\begin{align}\label{eq:Qsvd}
  Q := \argmax_{\bar{Q}^T\bar{Q} = \matI_d} \,\langle C, X\bar{Q} \rangle\,,
\end{align}
so that the orthogonal matrix $Q$ is chosen such that the elements of the rotated $XQ$ are as large as possible at the active positions $C$.
This problem can be solved by setting $Q=\bar{U} \bar{V}^T$, for $\bar{U}\bar{\Sigma}\bar{V}$ being the singular value decomposition (SVD) of $X^T C$.
For example, in the case of full matchings, when using $C = [\matI_{d}, \zerovec_{d,m{-}d}]^T$, the diagonal elements of the first block $X_1$ of $X$ are activated. With such a choice of $C$ we obtain $X^TC = X_1^T$, which corresponds to the rotation approach in~\cite{Pachauri:2013wx} with an additional SVD-based orthogonalisation of $X_1^T$. 
The important difference that makes our approach applicable to \emph{partial} matchings is that we successively construct the matrix $C$, rather than activating elements of a single block $X_i$ for some fixed $i$. 
The \emph{Successive Block Rotation Algorithm (\textsc{SBRA})} is summarised as follows: 
\begin{enumerate}[label=(\roman*)]
  \item First, we initialise $C$ to contain an $m_{\ell}{\times}m_{\ell}$ identity matrix in the $\ell$-th block, where $\ell = \argmax_{i} m_i$. All other elements of $C$ are zero.
  \item Given $C$, we obtain $Q$ by solving Problem~\eqref{eq:Qsvd}.
  \item Based on $X_{\text{rot}} = XQ$, we update $C$ so that the inactive rows of $X_{\text{rot}}$ chosen from the block with the largest number of inactive rows will be activated in the next step, as well as all active elements remain active.
\end{enumerate}
Step (ii) and (iii) are repeated until there are no further elements of $X_{\text{rot}}$ that shall become activated. We illustrate our algorithm in Fig.~\ref{fig:sbra}.
}

\newcommand{\minipageWidth}{.24} 

\subsection{Initialisation}\label{sec:nmfInit}
\begin{figure} %
\centering
	\includegraphics[scale=1]{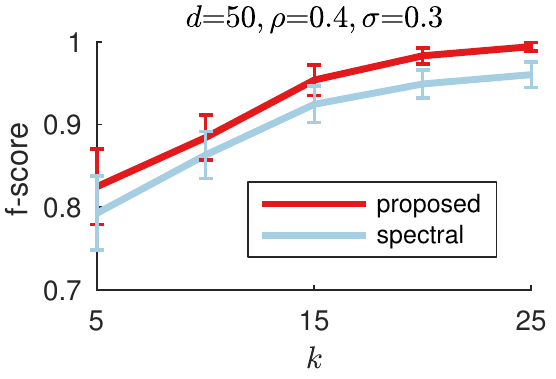}
	\vspace{-6mm}
	{\caption{Comparison of proposed vs.~spectral initialisation (cf.~Sec.~\ref{sec:synthData} for details).}
	\label{fig:spectralVsProposed} }
\end{figure}
Since Problem~\eqref{eq:nmfOpt} is non-convex, the initialisation of the matrices $V$ and $H$ plays a crucial role. We propose to initialise $V$ and $H$ based on a rotation of the spectral factorisation of the pairwise matching matrix ${W}$. %
Hence, we first compute the best rank-$d$ approximation of $W$ using eigendecomposition, so that $W \approx X X^T$, where $X \in \R^{m \times d}$ is the matrix of the (scaled) most dominant eigenvectors of $W.$
Subsequently, %
we rotate the columns of $X$  with $Q$, %
so that it becomes closer to $\mathcal{U}$, as described in Sec.~\ref{sec:projAfterRot}. 
Since we use an NMF algorithm based on multiplicative updates (cf.~Sec.~\ref{sec:alg}) that requires a non-negative initialisation, we set $V = (XQ)_+$ and $H = (XQ)_+^T$. 
In Fig.~\ref{fig:spectralVsProposed} we demonstrate that the proposed initialisation procedure is superior compared to using the spectral initialisation $X_+$.

\subsection{Projection onto $\mathcal{U}$}\label{sec:proj}
After solving Problem~\eqref{eq:nmfOpt} (with the algorithm described in Sec.~\ref{sec:alg}), we perform a projection-after-rotation,~i.e.~we find $Q$ {based on the \textsc{SBRA}} (Sec.~\ref{sec:projAfterRot}), and then project $V{Q}$ onto $\mathcal{U}$ to obtain $U$. 
This is done by solving $k$ (independent) linear assignment problems via the Auction algorithm \cite{Bertsekas:1998vt,Bernard:2016tv}. 
Moreover, similarly to existing approaches (e.g.~\cite{zhou2015multi,Maset:YO8y6VRb}), we prune bad matchings. 
To this end, we define a threshold $\theta \geq 0$ and remove all {multi-}matchings in $U$ where $V{Q} \odot U$ is smaller than $\theta$, for $\odot$ denoting the Hadamard product. {In order to ensure that $U \onevec_d = \onevec_m$, for each individual matching that is removed from a column of $U$, we add a new column to $U$ that contains all zeros apart from a single element being one---as such, this feature is now matched to its own universe feature (in the clustering perspective, it is a cluster comprising a single element, cf.~Fig.~\ref{fig:illustration}).}

\subsection{Algorithm}\label{sec:alg}
We call the overall synchronisation procedure \textsc{NmfSync}, which is summarised in Algorithm~\ref{nmfSync}. \textsc{NmfSync} comprises the following main steps: (i) initialisation of $V$ and $H$ (Sec.~\ref{sec:nmfInit}), (ii) minimisation of Problem~\eqref{eq:nmfOpt}, (iii) projection of $V$ onto $\mathcal{U}$ to obtain $U \in \mathcal{U}$ (Sec.~\ref{sec:proj}), and (iv) computation of the synchronised $W^{\text{sync}} = UU^T$.
\begin{algorithm}[h!]\label{nmfSync}
\scriptsize
\SetKwInput{Input}{Input}
\SetKwInput{Output}{Output}
\SetKwInput{Initialise}{Initialise}
\SetKwBlock{Repeat}{Repeat}{}

 \Input{$W \in \R^{m \times m}, d, \theta$}
 \Output{synchronised $W^{\text{sync}}$}%
\tcp{find best rank-$d$ approximation of $W$ (spectral method \cite{Pachauri:2013wx,Maset:YO8y6VRb})}
$[X,\Lambda] \leftarrow \operatorname{eig}(W,d)$,
$X \leftarrow X \Lambda^{0.5}$\\
\vspace{0.3mm}
\tcp{initialise according to Secs.~\ref{sec:projAfterRot}~and~\ref{sec:nmfInit}}
{$Q \leftarrow \operatorname{SBRA}(X)$}, $V \leftarrow (XQ)_+$, $H \leftarrow V^T$\\
\vspace{0.3mm}
\Repeat{\tcp{multiplicative updates of NMF \cite{Berry:2007jw}, $\epsilon > 0$ is a small number (numerics) }
  $H \leftarrow H \odot ((V^T W) \oslash ((V^TV)H + \epsilon))$ \tcp{$\oslash$ is element-wise division}
  $V \leftarrow V \odot ((W H^T) \oslash (W(HH^T) + \epsilon)$
}
\vspace{0.3mm}
\tcp{normalise so that the columns of $V$ and $H^T$ have the same $\ell_2$-norms}
$T \leftarrow \diag(\onevec_m^T (V \odot V))^{0.5}$,
$V \leftarrow  V T^{-1}$, $H = TH$\\
\vspace{0.3mm}
\tcp{project onto $\mathcal{U}$ according to Sec.~\ref{sec:proj}}
{$Q \leftarrow \operatorname{SBRA}(V)$}\\
$U \leftarrow \proj_{\mathcal{U}}(V {Q})$ \tcp{project $V {Q}$ onto $\mathcal{U}$ by solving  $k$ independent LAPs}
$U \leftarrow \operatorname{prune}(VQ, U,\theta)$ \tcp{prune uncertain matchings}
\vspace{0.3mm}
\tcp{compute synchronised $W$}
$W^{\text{sync}} \leftarrow UU^T$
 \caption{\textsc{NmfSync}}
\end{algorithm}

\section{Experiments}
In this section we evaluate the robustness of \textsc{NmfSync} and compare it against existing permutation synchronisation approaches.
To be more specific, we consider the \textsc{Spectral} method \cite{Pachauri:2013wx}, as implemented by the authors of \cite{zhou2015multi} to handle partial matchings based on a greedy rounding procedure, the \textsc{MatchEig} method \cite{Maset:YO8y6VRb}, and the \textsc{MatchALS} method \cite{zhou2015multi}. In our experiments we first consider synthetic data in a wide range of different configurations, followed by experiments on real data. We quantify the consistency of the pairwise matchings using the \emph{cycle-error} 
\begin{align}
  e_{\text{cycle}}(W) = \frac{1}{k^3}\sum_{i,j,\ell \in [k]} \| (P_{i\ell})_{R_{i\ell},:} (P_{\ell j})_{:,C_{\ell j}} - (P_{ij})_{R_{i\ell},C_{\ell j}}\|_F \,,
\end{align}
where for $i,j \in [k]$, the sets $R_{ij} \subseteq [m_i]$ and $C_{ij} \subseteq [m_j]$ denote the indices of non-zero rows and columns of $P_{ij}$, respectively. We use the ground truth error $e_{\text{gt}}$ (\emph{gt-error}) to measure the discrepancy between a given $W$ and the ground truth pairwise matchings $W_{\text{gt}}$, which we define as
  $e_{\text{gt}}(W) = \| W - W_{\text{gt}} \|_F$.
The $\text{f-score} = \frac{2\cdot\text{precision} \cdot \text{recall}}{\text{precision} + \text{recall}}$ summarises the \emph{precision} and \emph{recall}.

\subsection{Synthetic Data}  
\label{sec:synthData}
For our synthetic data experiments we generate the pairwise matchings $W$ for a given number of objects $k$, the universe size $d$, the observation rate $\rho$, and the error rate $\sigma$ as follows: For each $i \in [k]$, we first sample a random {(full)} permutation matrix {$P_i \in \perm_{d}$}. {To obtain a \emph{partial} permutation,  we remove each row of $P_i$ with probability $1{-}\rho$. As such, the number $m_i$ is implicitly determined by $\rho$, where the average of the $m_i$ is $\bar{m} = \rho d$.} Eventually, the ground truth matrix of cycle-consistent matchings is obtained as $W_{\text{gt}} = [P_{ij}]_{i,j \in [k]} = [P_i P_j^T]_{i,j \in [k]}$. We obtain the noisy matrix of pairwise matchings $W$ by perturbing each block $P_{ij}$ of $W_{\text{gt}}$ individually by randomly selecting a proportion of $\sigma$ of the rows of $P_{ij}$,
and then shuffle the selected rows. Note that %
 we perturb $W_{\text{gt}}$ in a symmetric fashion.
For each evaluated configuration, we draw $100$ samples of $W$ and report the averaged results.

\paragraph{Sensitivity Analysis:}
In Fig.~\ref{fig:sensitivityAnalysis} we present results of our sensitivity analysis with respect to the choice of the threshold parameter $\theta$, as well as to the choice of the estimate of the universe size $d$ that is used as additional input to all the methods. For a wide range of thresholds $\theta$ our method results in a smaller \emph{gt-error} compared to the other methods while providing \emph{cycle-consistent} results. Moreover, our method outperforms the other methods for varying universe sizes $d$ .

\newcommand{\figScaleB}{.7}
\begin{figure*} %
  \centerline{ \subfigure{\includegraphics[scale=\figScaleB]{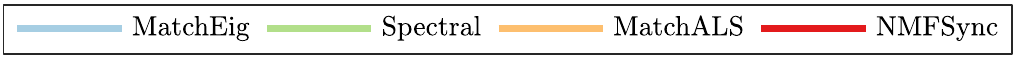} }}
  \vspace{-2mm}
     \centerline{%
      \rotatebox[origin=l]{90}{~~~\scriptsize{sensitivity to $\theta$}} \hfil
        \subfigure{\includegraphics[scale=\figScaleB]{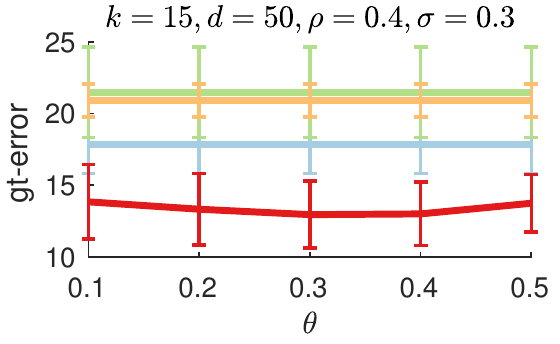}} \hfil %
        \subfigure{\includegraphics[scale=\figScaleB]{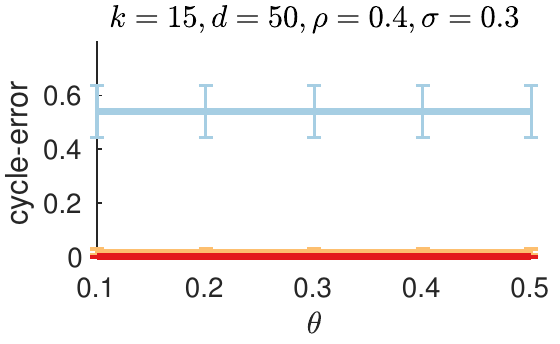}} \hfil %
        \subfigure{\includegraphics[scale=\figScaleB]{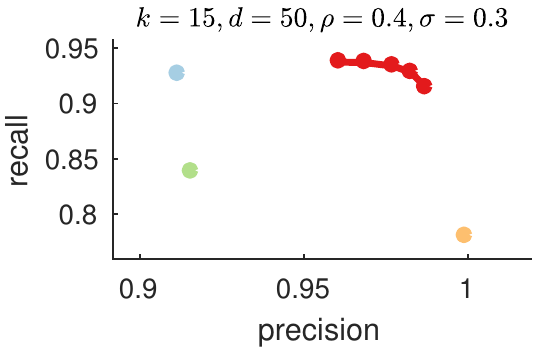}} \hfil %
      } 
      \vspace{-2mm}
      \centerline{%
      \rotatebox[origin=l]{90}{~~~~\scriptsize{sensitivity to $d$}} \hfil
        \subfigure{\includegraphics[scale=\figScaleB]{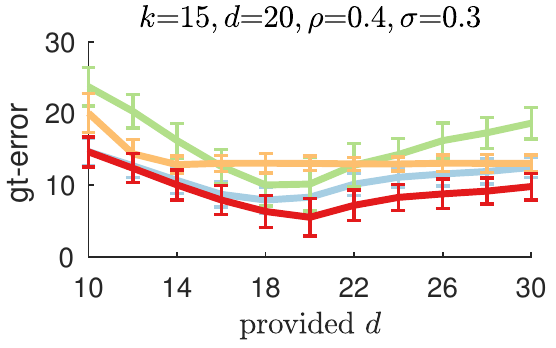}} \hfil %
        \subfigure{\includegraphics[scale=\figScaleB]{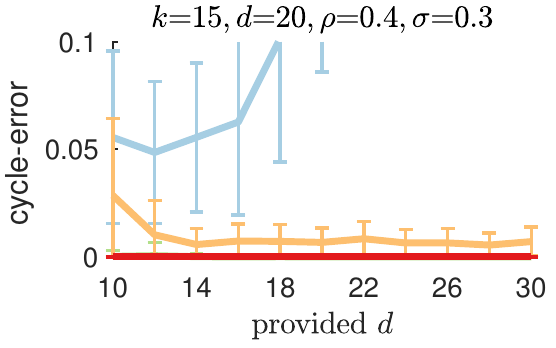}} \hfil %
        \subfigure{\includegraphics[scale=\figScaleB]{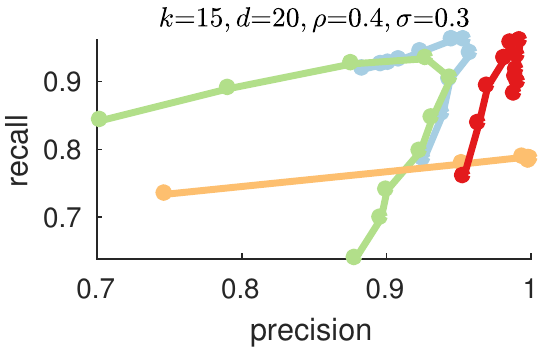}} \hfil %
      }%
      \vspace{-5mm}
    {\caption{Analysis of the sensitivity of our method to the choice of the threshold parameter $\theta \in \{0.1, 0.2, 0.3, 0.4, 0.5\}$ (top row), and the sensitivity of all methods to the provided universe size $d \in \{10, 12, \ldots, 30\}$ (bottom row, in this case the true $d$ is $20$).}
    \label{fig:sensitivityAnalysis}
    }
\end{figure*}

\paragraph{Comparison to Other Methods:}
\newcommand{\figScaleSynth}{.7}
\begin{figure*}[t!]%
  \centerline{\subfigure{\includegraphics[scale=\figScaleSynth]{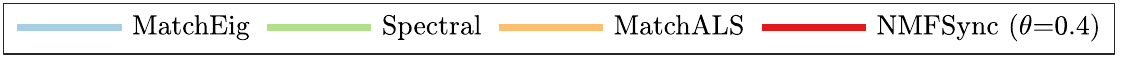}}}%
  \vspace{-3mm}
     \centerline{%
        \rotatebox[origin=l]{90}{$\quad$\scriptsize{cycle-error}} \hfil
        \subfigure{\includegraphics[scale=\figScaleSynth]{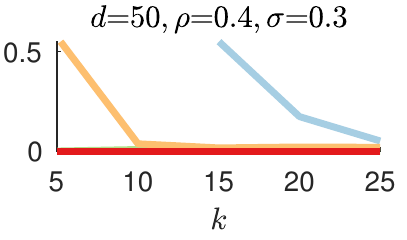}} \hfil
        \subfigure{\includegraphics[scale=\figScaleSynth]{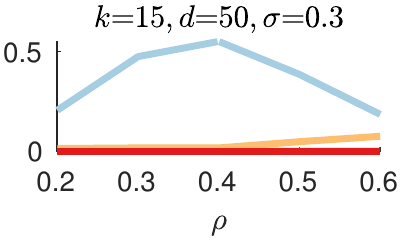}} \hfil
        \subfigure{\includegraphics[scale=\figScaleSynth]{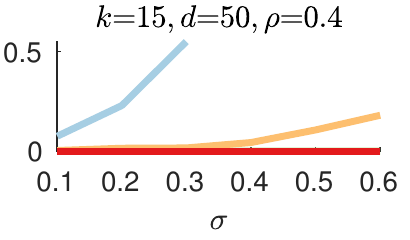}} \hfil
        \subfigure{\includegraphics[scale=\figScaleSynth]{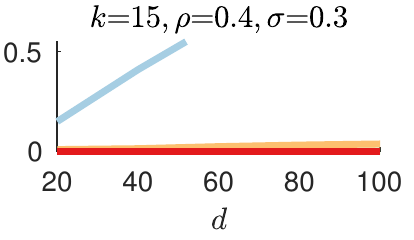}}
     }
     \vspace{-2mm}
      \centerline{%
       \rotatebox[origin=l]{90}{$\quad~~$\scriptsize{gt-error}} \hfil
        \subfigure{\includegraphics[scale=\figScaleSynth]{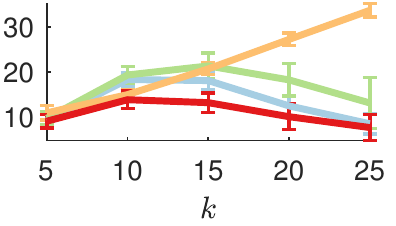}} \hfil
        \subfigure{\includegraphics[scale=\figScaleSynth]{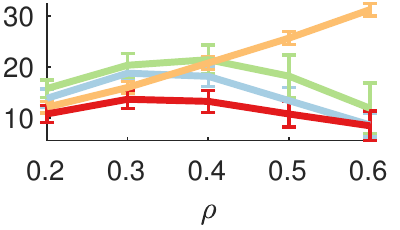}} \hfil
        \subfigure{\includegraphics[scale=\figScaleSynth]{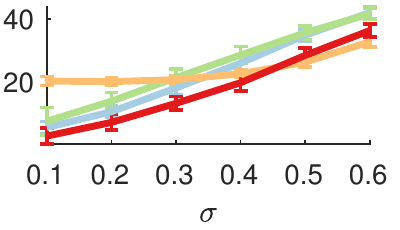}} \hfil
        \subfigure{\includegraphics[scale=\figScaleSynth]{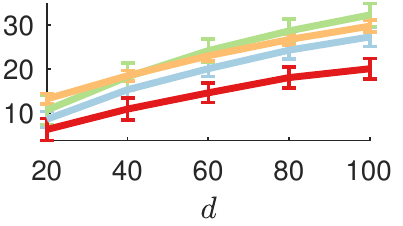}}
     }
      \vspace{-2mm}
      \centerline{%
       \rotatebox[origin=l]{90}{$\quad~~$\scriptsize{f-score}} \hfil
        \subfigure{\includegraphics[scale=\figScaleSynth]{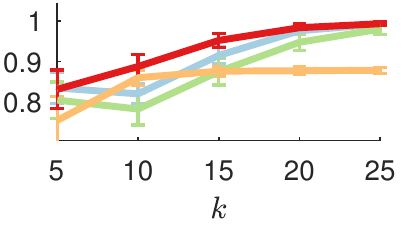}} \hfil
        \subfigure{\includegraphics[scale=\figScaleSynth]{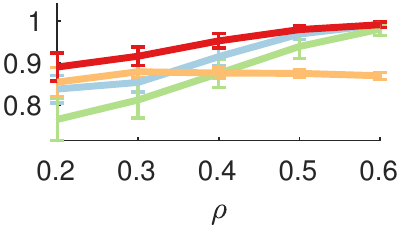}} \hfil
        \subfigure{\includegraphics[scale=\figScaleSynth]{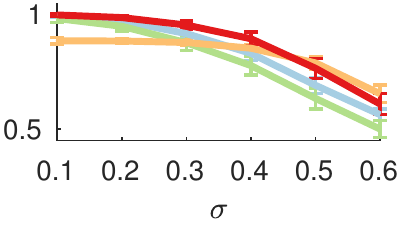}} \hfil
        \subfigure{\includegraphics[scale=\figScaleSynth]{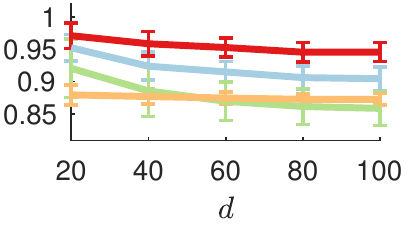}}
     }
     \vspace{-3mm} 
     \centerline{%
         \rotatebox[origin=l]{90}{$\quad$\scriptsize{\#matchings}} \hfil
        \subfigure{\includegraphics[scale=\figScaleSynth]{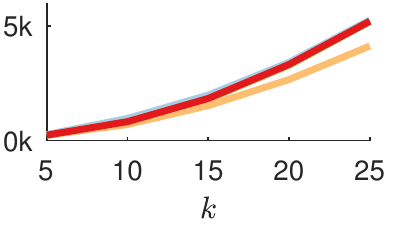}} \hfil
        \subfigure{\includegraphics[scale=\figScaleSynth]{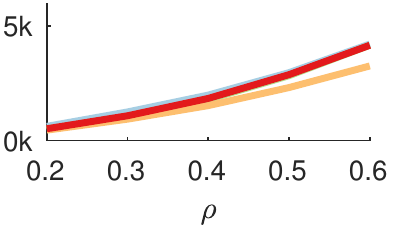}} \hfil
        \subfigure{\includegraphics[scale=\figScaleSynth]{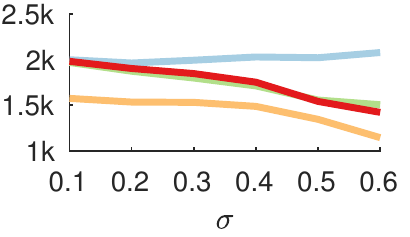}} \hfil
        \subfigure{\includegraphics[scale=\figScaleSynth]{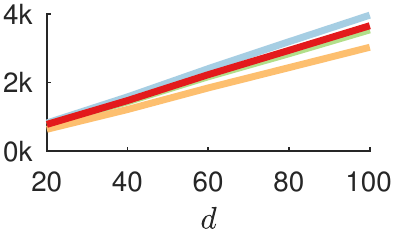}}
      }
      \vspace{-5mm}
    {\caption{Quantitative results for synthetic data for different varying parameters on the horizontal axis (the number of objects $k$, the observation rate $\rho$, the error rate $\sigma$, and the universe size $d$). The \emph{cycle-error}  of  \textsc{NmfSync} is always $0$. For the synthetic data experiments, the cycle-error of \textsc{Spectral} is also $0$.
    Considering the \emph{gt-error} (or analogously the \emph{f-score}) and the \emph{cycle-error}, \textsc{NmfSync} is clearly superior compared to its competitors.}
    \label{fig:syntheticResults}}
\end{figure*} 
The results of this experiments are shown in Fig.~\ref{fig:syntheticResults}, where the rows show the \emph{cycle-error}, the \emph{gt-error},  the \emph{f-score}, and the number of matchings (\emph{$\#$matchings}); and the columns show four different evaluation scenarios where in each scenario a different parameter varies along the horizontal axis. While \textsc{MatchEig} and \textsc{MatchALS} generally result in a non-zero \emph{cycle-error},~i.e.~the matchings are \emph{not cycle-consistent}, the \textsc{NmfSync} method guarantees cycle-consistent matchings. It can be seen that the overall result quality of \textsc{NmfSync} is superior compared to the other methods.

\subsection{Real Data}
\newcommand{\figScale}{.55}
\newcommand{\figScaleC}{.8}
\begin{figure*} %
  \centerline{ \subfigure{\includegraphics[scale=\figScale]{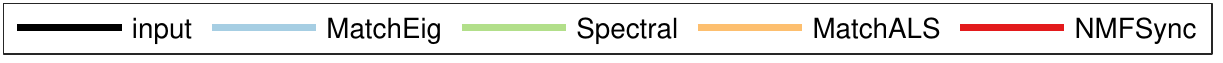} }}
  \vspace{-2mm}
     \centerline{%
     \rotatebox[origin=l]{90}{$\qquad~~$\scriptsize{FCM}} \hfil
        \subfigure{\includegraphics[scale=\figScale]{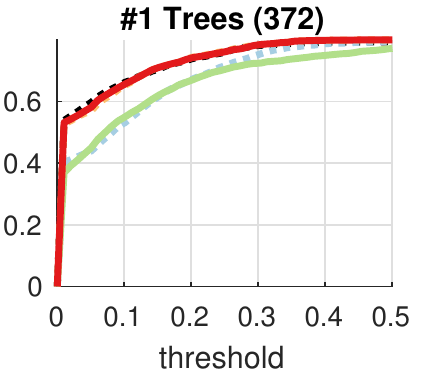}}%
        \subfigure{\includegraphics[scale=\figScale]{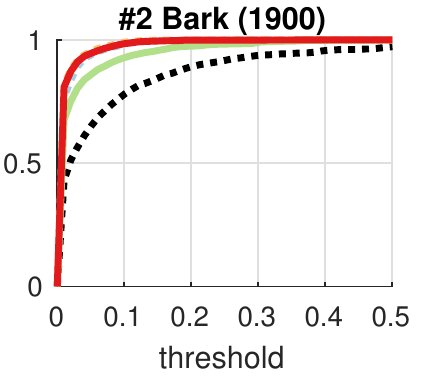}}\hfil%
        \subfigure{\includegraphics[scale=\figScale]{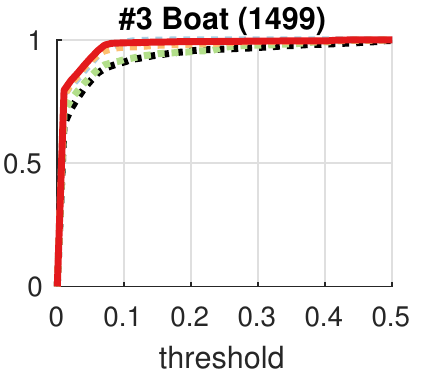}}\hfil%
        \subfigure{\includegraphics[scale=\figScale]{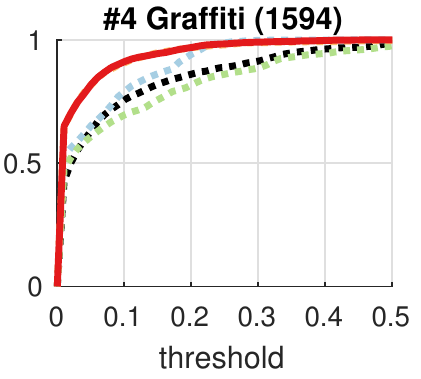}}\hfil%
        \subfigure{\includegraphics[scale=\figScale]{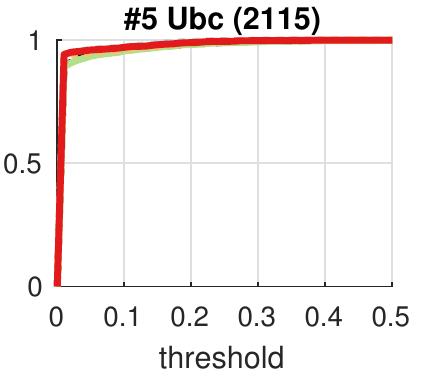}}\hfil%
        \subfigure{\includegraphics[scale=\figScale]{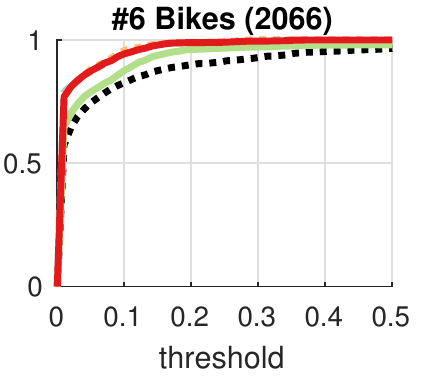}}\hfil%
      }
      \vspace{-3mm}
    \centerline{%
    \rotatebox[origin=l]{90}{$\qquad~~$\scriptsize{FCM}} \hfil
        \subfigure{\includegraphics[scale=\figScale]{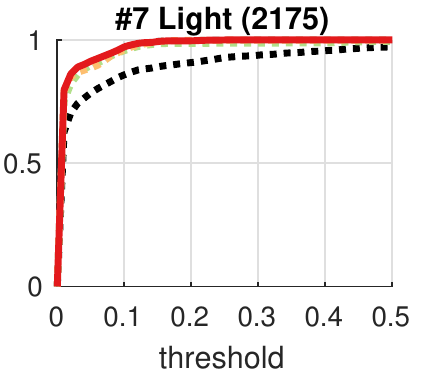}}\hfil
        \subfigure{\includegraphics[scale=\figScale]{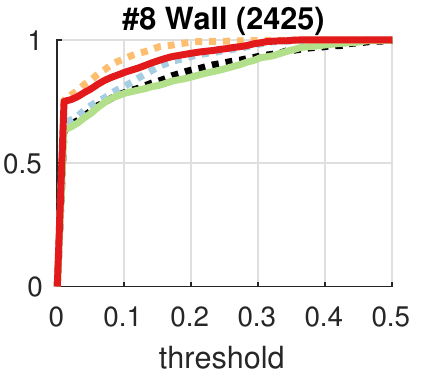}}\hfil 
        \subfigure{\includegraphics[scale=\figScale]{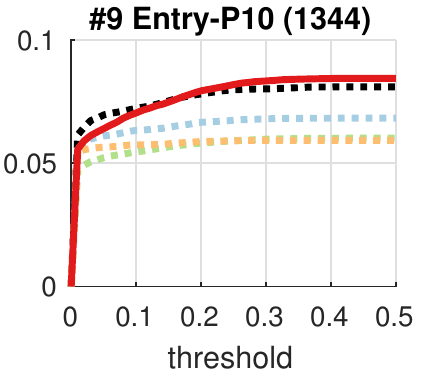}}\hfil \hfil
        \subfigure{\includegraphics[scale=\figScale]{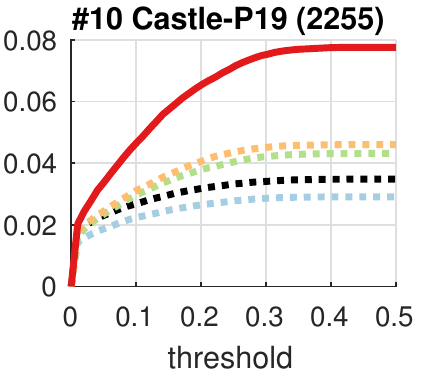}}\hfil
        \subfigure{\includegraphics[scale=\figScale]{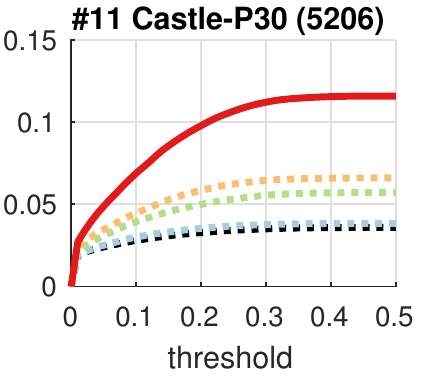}}\hfil 
        \subfigure{\includegraphics[scale=\figScale]{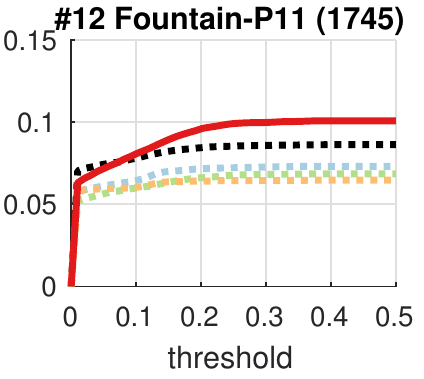}}%
     }
     \vspace{-3mm}
     \centerline{%
     \rotatebox[origin=l]{90}{$\qquad~~$\scriptsize{FCM}} \hfil
      \subfigure{\includegraphics[scale=\figScale]{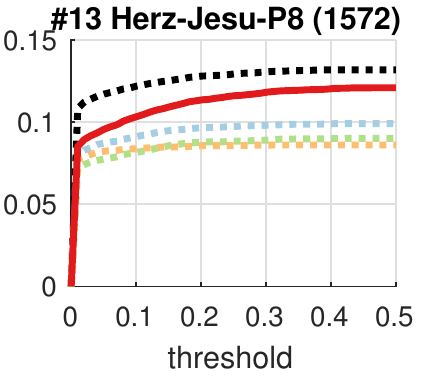}}\hfil%
      \subfigure{\includegraphics[scale=\figScale]{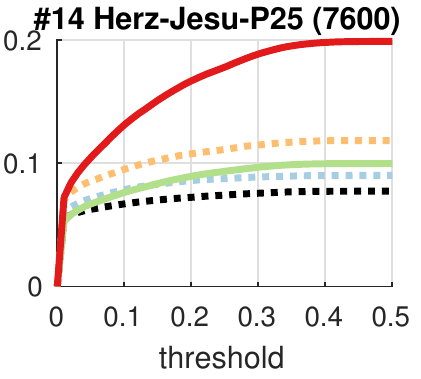}}\hfil%
      \subfigure{\includegraphics[scale=\figScale]{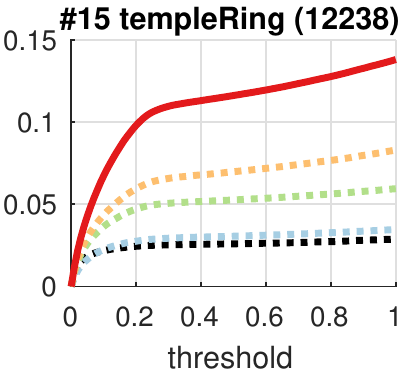}}\hfil%
      \subfigure{\includegraphics[scale=\figScale]{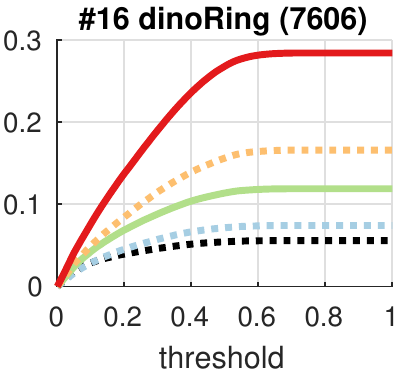}}\hfil%
      \subfigure{\includegraphics[scale=\figScale]{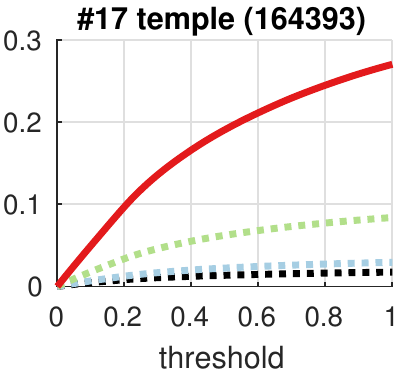}}\hfil%
      \subfigure{\includegraphics[scale=\figScale]{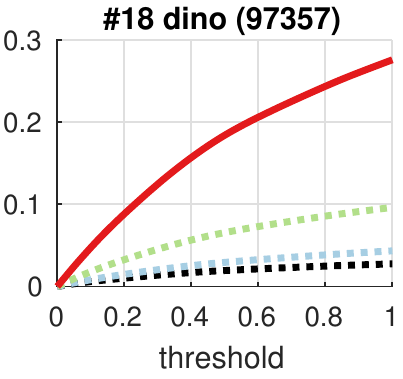}}%
       }
       \vspace{-3mm}
     \centerline{%
       \subfigure{\includegraphics[scale=\figScaleC]{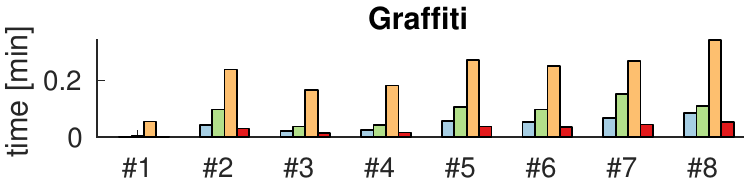}} \hfil %
       \subfigure{\includegraphics[scale=\figScaleC]{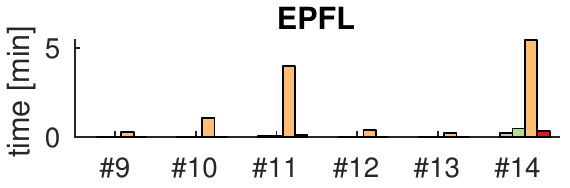}} \hfil %
       \subfigure{\includegraphics[scale=\figScaleC]{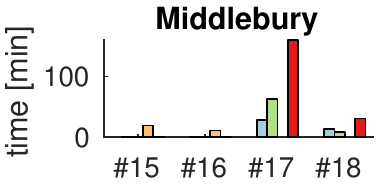}} %
       }
       \vspace{-5mm}
    {\caption{Results for the datasets \textsc{Graffiti} (\#1--\#8, $\theta {=} 0.4$),  \textsc{EPFL} (\#9--\#14, $\theta {=} 0$), and \textsc{Middlebury} (\#15--\#18, $\theta {=} 0$). Each plot shows the fraction of correct matchings (FCM) that have an error smaller than or equal to the threshold on the vertical axis (relative to the {largest image dimension}). The solid lines show results that are cycle-consistent, whereas the dashed lines show results that do not exhibit cycle-consistency. The title of each plot shows the size of the pairwise matching matrix $m$ in parentheses. 
    Considering FCM and cycle-error at the same time, \textsc{NmfSync} is superior compared to other approaches.}
    \label{fig:graffitiResults}}
\end{figure*} 

In our second set of experiments we consider real-world matching problems based on the \textsc{Graffiti} \cite{Mikolajczyk:2005ds}, \textsc{EPFL} \cite{Strecha:je} and the \textsc{Middlebury} \cite{Seitz:ft} datasets, all of which come with ground truth registrations. Our evaluations are based on the well-established protocol of \cite{zhou2015multi}, %
which was for example also used in \cite{Tron:kUBrCZhd}. To obtain the pairwise matchings $W$, we first extract SIFT features \cite{Lowe:2004kp} from the images, and then obtain the pairwise matchings based on simple nearest neighbour matching. Then, we use the so-obtained pairwise matchings as input to the synchronisation methods. We consider the \emph{fraction of correct matchings} (FCM), which indicates the fraction of matchings  that have an error less than a specified threshold. {Since the true number of correct matchings is unknown (cf.~\cite{Tron:kUBrCZhd}), the FCM is computed relative to the number of image features, as done in \cite{zhou2015multi}.}

\paragraph{Results:} In Fig.~\ref{fig:graffitiResults} we show quantitative results. The first three rows show the FCM
for the individual problem instances \#1 to \#18, where the solid lines indicate cycle-consistent results (\textsc{NmfSync}) and the dashed lines indicate cycle-inconsistent matchings (all other methods, with the exception of \textsc{Spectral} in a few instances). 
{Note that the multi-image matching problems in the \textsc{Graffiti} dataset are  easier compared to the \textsc{EPFL} and \textsc{Middlebury} datasets, as the overlap of the visible object parts in the \textsc{Graffiti} images are much larger. This also explains that the values of the (relative) FCM scores in the other two datasets are lower (in this case the number of features in an image is an overly conservative upper bound for the true number of matchings).}
Considering the FCM and cycle-consistency, \textsc{NmfSync} clearly outperforms the other methods.
For the moderately-sized problem instances \#1 to \#16, where $m$ is between $372$ and $12{,}238$, all methods have comparable runtimes, with the exception of \textsc{MatchALS} being substantially slower. Note that \textsc{MatchALS} cannot be used for processing the very large instances \#17 and \#18 due to its unscalability in terms of memory (cf.~Sec.~\ref{sec:discussion}).

\subsection{Discussion \& Limitations}\label{sec:discussion}
Due to the pruning of uncertain matchings in \textsc{NmfSync} based on $\theta$ (Sec.~\ref{sec:proj}), the total number of obtained matchings of \textsc{NmfSync} varies depending on the input quality. For example, the third column in Fig.~\ref{fig:syntheticResults} illustrates that when increasing the error rate while keeping other parameters fixed, the number of matchings returned by \textsc{NmfSync} decreases. This reflects that our method implicitly takes into account the larger input corruption. Note that other methods also prune uncertain matches.

While the Auction algorithm \cite{Bertsekas:1998vt} for solving the LAP has (roughly) cubic \emph{worst-case} complexity \cite{Bertsekas:1988cr}, the analysis in \cite{Schwartz:1994db} suggests that the \emph{average} complexity is in the regime {$\mathcal{O}(d^2 \log d)$. Our rotation scheme involves the computation of an SVD with complexity $\mathcal{O}(d^3)$. Both, the LAP and the SVD are solved $\mathcal{O}(k)$ times. 
We have observed that the spectral decomposition and the NMF algorithm, with \emph{per-iteration} complexity $\mathcal{O}(m^2d)$, usually dominate the overall runtime.}
 In contrast to \textsc{MatchALS}, our method never requires the computation of the dense and large $m \times m$ matrix $VH$ (cf.~Alg.~\ref{nmfSync}), such that \textsc{NmfSync} is much more memory efficient. With that, our method is able to handle very large problem instances, as we show in Fig.~\ref{fig:graffitiResults} for instances \#17 and \#18, where $m$ goes up to ${\approx}160{,}000$.

{One property that is common to all existing synchronisation methods is that they only consider given (partial) matchings without explicitly incorporating any \emph{higher-order information} (such as distances between pairs of features positions). While in certain applications ignoring higher-order information is desirable (e.g.~when it is simply not available), in other cases such information could be leveraged to obtain more reliable matchings. Hence, albeit being computationally challenging, we believe that the incorporation of higher-order terms (e.g.~in the spirit of the QAP) into synchronisation problem formulations is an interesting direction for future work.
}

\section{Conclusions}
Based on a non-negative factorisation of the matrix of pairwise matchings, we have presented the \textsc{NmfSync} method for the synchronisation of partial permutation matrices. {We have found that even though the ground truth pairwise matching matrix $W$ is symmetric, from a computational perspective it is actually beneficial to perform an unsymmetric factorisation (cf.~Fig.~\ref{fig:symvsnonsym}).}
In order to deal with the non-convexity of our formulation, we have proposed a novel scheme for rotating the solution of the spectral relaxation such that it %
provides a suitable initialisation for the NMF. %
Moreover, we have generalised the projection-after-rotation approach of the \textsc{Spectral} method~\cite{Pachauri:2013wx}, so that it can handle \emph{partial} matchings (Sec.~\ref{sec:proj}). 
In contrast to the \textsc{MatchALS} method~\cite{zhou2015multi}, and the more recent \textsc{MatchEig} method~\cite{Maset:YO8y6VRb}, our approach is guaranteed to produce a cycle-consistent solution. Since cycle-consistency is an intrinsic property of the (unknown) true matchings, we argue that it is important to achieve.
{Furthermore, we have demonstrated that \textsc{NmfSync} is comparable to existing methods in terms of scalability, and that it quantitatively outperforms existing approaches on various datasets. Due to these favorable properties, we believe that \textsc{NmfSync} is a significant contribution towards the (sub)field of (partial) permutation synchronisation.
}

$\newline$
\noindent\paragraph{Acknowledgements:}
This work was funded by the ERC Starting Grant CapReal~(335545), the ERC Consolidator Grant 4DRepLy (770784), and by the Luxembourg National Research Fund (FNR, C14/BM/8231540).

\section*{References}
\bibliographystyle{elsarticle-harv}
\bibliography{extracted}

\end{document}

%% file: macros.tex
\newcommand{\proj}{\operatorname{proj}}

\newcommand{\R}{\mathbb{R}}

\newcommand{\N}{\mathbb{N}}

\newcommand{\G}{\mathcal{G}}

\DeclareMathOperator*{\argmin}{arg\,min}
\DeclareMathOperator*{\argmax}{arg\,max}

\newtheorem{theorem}{Theorem}
\newtheorem{lemma}[theorem]{Lemma}
\newtheorem{definition}[theorem]{Definition}

\newcommand{\perm}{\mathbb{P}}

\newcommand{\diag}{\operatorname{diag}}
\newcommand{\matI}{\mathbf{I}}

\newcommand{\onevec}{\mathbf{1}}
\newcommand{\zerovec}{\mathbf{0}}

\renewcommand{\paragraph}{\textbf}